\documentclass{article}

\RequirePackage[OT1]{fontenc}
\RequirePackage{amsthm,amsmath}
\RequirePackage[round]{natbib}

\usepackage{pgfplots}
\pgfplotsset{compat=newest}
\usepackage{amsthm,amsmath}
\usepackage{graphicx}
\usepackage{url}
\usepackage[]{algorithm2e}
\usepackage{multirow}
\usepackage{balance}
\usepackage{multirow}
\usepackage{url}
\usepackage{authblk}

\usepackage[figuresright]{rotating}

\newtheorem{prop}{Proposition}

\begin{document}

\title{Optimal estimates for short horizon\\ travel time prediction in urban areas}

\author[1]{Indr\.e \v{Z}liobait\.e}
\author[2]{Mikhail Khokhlov}
\affil[1]{Aalto University and HIIT, e-mail: \texttt{indre.zliobaite@aalto.fi}}
\affil[2]{Yandex,  e-mail: \texttt{aeol@yandex-team.ru}}

\date{}
\maketitle 

\begin{abstract}
Increasing popularity of mobile route planning applications based on GPS technology provides opportunities for collecting traffic data in urban environments. 
One of the main challenges for travel time estimation and prediction in such a setting is how to aggregate data from vehicles that have followed different routes, and predict travel time for other routes of interest. One approach is to predict travel times for route segments, and sum those estimates to obtain a prediction for the whole route. 
We study how to obtain optimal predictions in this scenario. 
It appears that the optimal estimate, minimizing the expected mean absolute error, is a combination of the mean and the median travel times on each segment, 
where the combination function depends on the number of segments in the route of interest. 
We present a methodology for obtaining such predictions, and demonstrate its effectiveness with a case study using travel time data from a district of St. Petersburg collected over one year. 
The proposed methodology can be applied for real-time prediction of expected travel times in an urban road network. 
\end{abstract}
%


\section{Introduction}

Traffic congestions are very common in modern urban environments. 
Higher than ever penetration of mobile sensing technologies allows collecting real time traffic data from many users at relatively low costs. 
Such data can be used for providing real time information about traffic conditions.
In addition, data accumulated over time can be used for modeling traffic patterns, 
and making predictions about traffic in the nearest future that can help in travel planning for individuals, and contribute towards mitigating traffic congestions. 

One of the main challenges in urban travel time prediction is that vehicles follow different routes, and traffic conditions are rapidly changing. Suppose we are interested to predict current travel time for a  given route. Probably very few or no drivers has recently traveled exactly the same route, so no direct data is available for making predictions. 
However, it is very likely that a number of vehicles have passed through different segments on the route of interest, and we have data on recent travel times on separate segments. 
We study how to optimally combine predictions made on individual segments into a prediction for the whole route of interest. 
It turns out that a simple sum of predictions for individual segments is not optimal due to the nature of travel time distribution. 
We propose a methodology for estimating travel time on individual segments,  and how to combine these estimates in an optimal way that would minimize the mean absolute error of travel time prediction for any route in the road network. 
We demonstrate the effectiveness of the proposed methodology with a case study using travel time data from a district of St. Petersburg city collected over one year. 

Several commercial services for predicting travel time exist, such as Yandex.Traffic\footnote{\url{http://maps.yandex.com/traffic}}; however, 
we are not aware of any publicly available research, that would systematically investigate, how to optimally aggregate predictions made for road segments.




The paper is organized as follows. 
Section \ref{sec:background} presents background and related work. 
In Section \ref{sec:criteria} we discuss alternative optimization criteria for travel time prediction, and theoretically analyze expected prediction accuracies.
Section \ref{sec:method} presents our methodology for obtaining optimal travel time estimates. 
An experimental case study is presented in Section \ref{sec:experiments}. 
Section \ref{sec:conclusion} concludes the study.

\section{Background}
\label{sec:background}

This section overviews major research directions in travel time prediction from empirical data.

\subsection{Possible data sources}

Two types of traffic data sources can be distinguished: static and dynamic. 
Static data comes from sensors fixed on roads, such as, inductive loop detectors, or cameras for license plate recognition. 
Several stationary sensors, installed along a road, can estimate how long it takes for a vehicle to travel from one sensor to another. 

Dynamic data comes from sensors (typically GPS) installed in cars. 
Data can be collected by designated probe vehicles driving for the purpose of data collection, service fleet, or private vehicles driving on their own business (a.k.a. \emph{floating car}), equipped with GPS receivers.

Stationary data collection provides a complete traffic view, it counts all the vehicles, but it is relatively expensive to deploy, and it is not suitable for tracking vehicles in urban environments, where there are lots of small streets and possible turns. A large number of stationary sensors would be necessary for following vehicles. 
On the other hand, dynamic data collection, using GPS tracking, can track the exact movement of a vehicle no matter how many possible turns there are, but cars that have the necessary equipment can be tracked. 
Dynamic data captures only a sample, not a complete traffic.

We study travel time prediction in urban environments, hence, we focus on dynamically collected data, and methods for working with such data.

\subsection{Related work}

Travel time prediction from empirical trip data has been studied for over a decade.
Table \ref{tab:papers} presents a summary of representative work in this area.

\begin{sidewaystable}
\caption{Summary of related work.}
\centering
\begin{tabular}{lllll}
\hline
Study&Road&Data source&Predictive models&Evaluation measures\\
\hline
\citet{Kwon00}&highway&loop detectors & stepwise regression, tree, ANN & MSPE\\
\citet{Ishak02}&highway& loop detectors  & GLM & MAPE\\
\citet{Chien03}&highway& RFT probe vehicles & Kalman filter & MAPE, RMSEP\\
\citet{Rice04}&highway& loop detectors & linear regression, kNN & RMSE\\
\citet{Wu04}&highway& loop detectors & SVM regression & MAPE, RMSEP\\
\citet{Bajwa05}&highway& loop detectors, cameras & pattern matching & correlation, RMSE, hit ratio\\
\citet{Innamaa05}&highway& loop detectors, cameras & ANN & MAE, RMSE, ME,\\
& & & & MRE, hit ratio\\
\citet{Guin06}&highway& cameras & ARIMA & MAE, MAPE, RMSEP\\
\citet{Fei11}&highway& loop detectors & Bayesian & MAE, MAPE, RMSE\\
\citet{Heilmann11}&highway& local detector, toll &  kernel predictor & RMSE\\
\hline
\citet{Fabritiis08}&city&  GPS private cars & pattern matching and ANN & MAPE, RMSE\\
\citet{Vanajakshi08}&city& GPS busses, probe cars & Kalman filter & MAPE\\
\citet{Markovic10}&city& GPS courier vehicles &  kNN and ARIMA & MAPE, ME, RMSE\\
\citet{Westgate13}&city& GPS ambulances &  Bayesian model  & RMSE\\
\citet{Jones13}&highway,& GPS floating car & SVM regression & MAPE\\
					&city& &  & \\
\hline
\end{tabular}
\label{tab:papers}
\end{sidewaystable}

Studies differ in data sources, in traffic environment, predictive models and evaluation measures used; however, 
the majority of studies focus on highways where data is collected via induction 
loop detectors \citet{Kwon00,Ishak02,Rice04,Wu04,Bajwa05,Innamaa05,Guin06,Fei11}, 
cameras \citet{Bajwa05,Innamaa05,Guin06}, radio-frequency (RF) identification tags \citet{Chien03}, or toll stations \citet{Heilmann11}.
In the highway settings forming the prediction target is straightforward, because most of the vehicles follow the same route, 
and plenty of historical data is available for modeling from the route in question. In these settings there is no need for aggregated predictions. 

The most popular measures for travel time prediction accuracy are the mean absolute percentage error (MAPE), 
which is a normalized version of the mean absolute error, and the root mean square error (RMSE), or its normalized version RMSEP.
Often in research studies accuracy is reported using several alternative measures in order to provide a more comprehensive view of the results.
Accuracy is measured over individual route, or segment. 
We are not aware of any research work investigating optimization criteria or evaluation measures for travel time prediction in a road network. 

The scenario considered by Fei et al \citet{Fei11} is to some extent related to our problem setting. 
The authors aggregate predictions for 66 segments of one highway. 
They use a simple sum of means as the combination rule. 
They do not investigate any alternative rules, and the focus of the paper is not on optimal aggregation methods, as is the focus of our paper. 
We will demonstrate that the mean rule is sub-optimal for combining a small number of segments, but approaches the optimum when the number of segments is large. 
Practically, 66 segments is already a large number, hence, the sum of means may work reasonably well in this case.

Several studies model travel times in urban environments using GPS data \citet{Fabritiis08, Vanajakshi08, Markovic10, Jones13}. 
Vanajakshi et al \citet{Vanajakshi08} predict bus travel times over a test route. The setting is similar to a highway setting, where all the vehicles follow the same route, thus, modeling data from the same route is directly available, and there is no need for aggregated predictions. 
Other three studies \citet{Fabritiis08, Markovic10, Jones13} use floating car data, where vehicles can follow many different routes. 
However, all three studies consider a simplified scenario, where predictions for individual segments are made and evaluated individually, there are no aggregated predictions for different routes. 
In comparison, our study considers a more advanced prediction scenario, where the goal is to optimize the prediction accuracy not over individual segments, but over a set of possible routes. We will demonstrate that the optimization criteria in those two scenarios is not the same. 

A study on ambulance arrival times \citet{Westgate13} uses a Bayesian model for estimating travel times over the road network.  The model parameters are learned all at once for the whole network. Learning such models requires a lot of training samples, which is not feasible in our case, where only a small fraction of all cars in the network provide data, and data distribution is changing over time, which would require different parametrisation at different times of day.

Finally, a different line of research develops traffic simulation models (see e.g. \citet{Treiber13}), which are mainly used for road planning, transportation logistics, car design and manufacturing, but to the best of our knowledge, such models are not used for real-time traffic predictions. 
One of the main limiting factors is that effective predictive models would need to know in advance at least where each vehicle is heading, which is practically infeasible. 


\subsection{Predicting travel time vs. predicting speed}

Majority of related studies aim at predicting travel time, only \citet{Heilmann11} 
considers speed prediction.
While speed is easier for humans to interpret (we usually think about traffic conditions in terms of speed, not travel time), time has an important advantage as a target variable for prediction.

One of the main purposes of traffic prediction is to plan optimal routes for vehicles driving in the city. 
Many criteria for route optimality can be considered, such as route length, quantity of fuel used and complexity of driving directions, but the most common by far is the total driving time. 
In a deterministic setting, travel time can always be computed from speed, but the task becomes more complex in a stochastic setting, where the expected time and the expected speed are not related by a strict dependence. 
A model, that is good at predicting expected speed, may be misleading if used for predicting travel time, as the following example illustrates. 

Suppose, a driver can take one of two possible routes ($A$ or $B$) of equal length $12$ km. 
 Due to traffic conditions (fast or slow traffic), two variants of travel time are possible on each route, and they may happen with equal prior probability.  
Traffic conditions on these routes are independent. 
All possible outcomes of the journey are indicated in Table \ref{tab:speeds}.
We can see that the expected speeds on both routes are equal, but the expected travel times differ. 
Hence, generally it is not possible to deduce expected travel time given only expected speed.
Therefore, travel time is chosen as the target variable given the task to plan the fastest route. 


\begin{table}[t]
\caption{Example: comparing travel speeds vs. travel times.}
\centering
\begin{tabular}{llccr}
\hline
& & Probability & Travel time & Speed\\
\hline
\multirow{2}{*}{Route A}  & Fast traffic & $1/2$ & $12$ min & $60$ km/h\\
& Slow traffic & $1/2$ & $24$ min & $30$ km/h\\
& \multicolumn{2}{r}{Expected values}& $\mathbf{18}$ min & $\mathbf{45}$ km/h\\
\hline
\multirow{2}{*}{Route B} & Fast traffic & $1/2$ & $10$ min & $72$ km/h\\
& Slow traffic & $1/2$ & $40$ min & $18$ km/h\\
& \multicolumn{2}{r}{Expected values}& $\mathbf{25}$ min & $\mathbf{45}$ km/h\\
\hline
\end{tabular}
\label{tab:speeds}
\end{table}



\section{Estimating travel time from historical data}
\label{sec:criteria}

The ultimate purpose of traffic information service is to help users to find an optimal route between two points at a give time. 
In the urban environment conditions of alternative routes are similar, hence, finding an optimal route typically resorts to finding the fastest route. 
A good traffic information service would predict travel times as accurately as possible for as many users as possible. 
Choosing the right optimization criteria in this scenario is not trivial. 
In this section we formally introduce the problem of travel time prediction, 
discuss alternative optimization criteria for travel time prediction, 
and theoretically analyze expected prediction accuracies.

\subsection{Travel time data distribution}
\label{sec:datadistribution}

Firstly, let us consider some characteristics of travel time data. 
The distribution of travel times is positively skewed. 
The travel time over any road segment (of a positive length) is always larger than zero.
Travel time approaches infinity when travel speed approaches zero, i.e. a vehicle drives very slowly.
Most of the probability mass is expected to be concentrated at small positive values.
With such a distribution the median of data is typically smaller than the mean. 

The log-normal distribution is a good example of such a distribution. 
If a random variable $x$ is log-normally distributed ($x \sim \mathit{ln} {\cal N}(\mu,\sigma)$), then $y = \log(x)$ has a normal distribution. 
Figure \ref{fig:lognorm} (a) presents example pdfs of the log-normal distribution, 
and (b) presents empirical distributions of travel times in one road segment in St. Petersburg observed in November-December 2012. 
We can see that the empirical traffic data distribution resembles log-normal distribution.
\pgfmathdeclarefunction{gauss}{2}{%
  \pgfmathparse{1/(#2*x*sqrt(2*pi))*exp(-((ln(x)-#1)^2)/(2*#2^2))}%
}
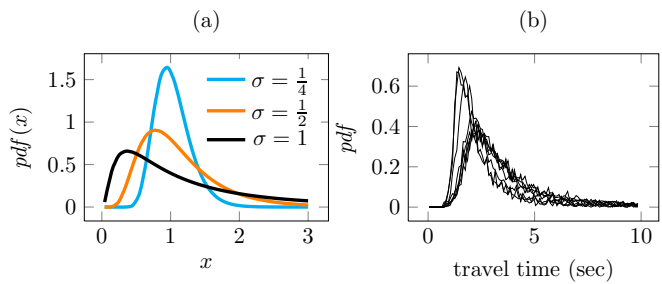
\begin{figure}
\begin{tikzpicture}[scale = 0.85]
\begin{axis}[name=plot1,
every axis plot post/.append style={mark=none,domain=-2:3,samples=50,smooth, line width = 1.5pt}, 
width=5.4cm, height=4.2cm,
title = (a),
xlabel = $x$,
ylabel = $\mathit{pdf}(x)$,
legend pos= north east,
legend style={draw=none}]
\addplot[cyan] {gauss(0,0.25)};
\addplot[orange] {gauss(0,0.5)};
\addplot[black] {gauss(0,1)};
\legend{$\sigma=\frac{1}{4}$, $\sigma=\frac{1}{2}$,$\sigma=1$} 
\end{axis}
\begin{axis}[name=plot2, at=(plot1.right of south east), anchor=left of south west,
width=5.5cm, height=4.2cm, 
ylabel = $\mathit{pdf}$,
title = (b),
xlabel = travel time (sec)] 
\addplot[black, line width = 0.4pt]  table[x=step, y=count0] {times_distribution_statistics_NovDec_69258114_0.dat};
\addplot[black, line width = 0.4pt]  table[x=step, y=count3] {times_distribution_statistics_NovDec_69258114_0.dat};
\addplot[black, line width = 0.4pt]  table[x=step, y=count6] {times_distribution_statistics_NovDec_69258114_0.dat};
\addplot[black, line width = 0.4pt]  table[x=step, y=count9] {times_distribution_statistics_NovDec_69258114_0.dat};
\addplot[black, line width = 0.4pt]  table[x=step, y=count12] {times_distribution_statistics_NovDec_69258114_0.dat};
\addplot[black, line width = 0.4pt]  table[x=step, y=count15] {times_distribution_statistics_NovDec_69258114_0.dat};
\addplot[black, line width = 0.4pt]  table[x=step, y=count18] {times_distribution_statistics_NovDec_69258114_0.dat};
\addplot[black, line width = 0.4pt]  table[x=step, y=count21] {times_distribution_statistics_NovDec_69258114_0.dat};
\end{axis}
\end{tikzpicture}
\caption{(a) Example log-normal probability density functions (pdf), $\mu=0$. (b) Empirical distributions of travel times on one road segment, each line represents the same hour (e.g. 3AM, 6AM,...).}
\label{fig:lognorm}
\end{figure}

\subsection{Problem setting}


Suppose we have a fixed road network divided into segments (links). 
For simplicity assume that  each segment represents one-way traffic from crossing to crossing. 
Let ${\cal R} = \{r_1,r_2,\ldots,r_n\}$ be a set of road segments. 
For each segment $r$ we know its length $l_r$, and the neighboring segments to which $r$ is connected. 

In addition, for each segment we know travel times of vehicles that are using a mobile application for navigation
Travel times are extracted from GPS traces.
While dividing a road network into segments, and mapping GPS traces to the segments is a great challenge, 
it is out of the scope of the current paper, which focuses on data analysis for travel time prediction. 
An interested reader is referred to \citet{Lou09}, discussing some of these challenges. 
In our study we assume that the road segments and the travel times are readily available. 

In our setting vehicle information is anonymous, only travel times over a sequence of segments is available, referred to as a trip. 
It is not possible to know anything about the vehicle, or whether several trips originate from the same vehicle. 

Let $D = \{d_1,d_2,\ldots,d_m\}$ be a set of trips observed, and let $\delta_i^{(d)}$ be the index of the $i^{th}$ road segment in trip $d$.
Then a trip $d$ can be described as a sequence of road segments $\left(r_{\delta_1^{(d)}},r_{\delta_2^{(d)}},\ldots,r_{\delta_{k_d}^{(d)}}\right)$.

Let $t^{(d)}_i$ be the travel time on the $i^{th}$ segment in trip $d$, and $k_d$ be the number of segments in trip $d$, then the travel time for the whole trip $d$ is 
\begin{equation}
T^{(d)} = \sum_{i=1}^{k_d} t^{(d)}_i.
\label{eq:T}
\end{equation}

Our task is, given an intended trip as a sequence of segments, to predict the travel time from historical data. 
At least two practical application scenarios related to this task are possible. 
First, someone is interested to know how long it will take to travel from point A to point B. 
Second, a navigation system is selecting the fastest route from point A to point B from several alternative routes. 


Since no information about vehicles or drivers is accessible, we cannot make personalized predictions. 
We can only make predictions for a particular route at particular time, assuming an average driver. 
As mentioned earlier, casting predictions for all possible routes is impractical, and combinatorially infeasible on any realistically sized road network.
Thus, the best we can do is make predictions for individual road segments, and then aggregate them to get a prediction for the whole trip.
The prediction for trip $d$ would be
\begin{equation}
\hat{T}^{(d)} = \sum_{i=1}^{k_d} \tau(r_{\delta_i^{(d)}}),
\label{eq:That}
\end{equation}
where $\tau(r_j)$ is the predicted travel time for road segment $j$, and $\delta_i^{(d)}$ is the index of the $i^{th}$ road segment in trip $d$.
We will refer to this approach as \emph{additive prediction}.

\subsection{Optimization criteria}

Defining the optimization criteria, and selecting an informative evaluation measure in the additive prediction scenario is not trivial.
From the route planning perspective,  for any new trip $d$ the predicted travel time $\hat{T}^{(d)}$ should be as accurate as possible. 
Many alternative accuracy measures can be considered, as seen in Table~\ref{tab:papers}.

Root Mean Squared Error (RMSE) is a popular loss function for predictive modeling because of its convenient analytical properties. 
RMSE is the square root of the Mean Squared Error (MSE), which is defined as
\begin{equation}
\mathit{MSE} = \frac{1}{m}\sum_{d=1}^{m} (\hat{T}^{(d)} - T^{(d)})^2, 
\end{equation}
where $\hat{T}$ denotes the predicted value $T$ denotes the true observed value, and $m$ is the number of trips, on which evaluation is made. 
The lower MSE/RMSE, the better. MSE/RMSE punish large deviations of predictions from the true values.

Mean Absolute Error (MAE) is an alternative measure, defined as 
\begin{equation}
\mathit{MAE} =  \frac{1}{m}\sum_{d=1}^{m} |\hat{T}^{(d)} - T^{(d)}|, 
\label{eq:mae}
\end{equation}
the lower, the better.

While opimizing RMSE would minimize large deviations from the truth, minimizing MAE would give a larger number of correct predictions to more drivers, and, thus, would be more valuable for more customers.

More formally, suppose we have two alternative routes $A$ and $B$ with travel times $T^A$ and $T^B$ independently distributed with different distributions. 
We would like to choose the one that is more likely to be the fastest. 
It can be shown that if $T^A \sim F( {\cal N}(\mu_A,\sigma_A))$ and $T^B \sim F( {\cal N}(\mu_B,\sigma_B))$, where $F$ is some monotone function, 
then we should choose the route with the lowest median travel time (a proof can be found in the Appendix, Proposition~\ref{pro:monotone}).
For predicting the median we need to use MAE criterion, as it will be demonstrated in the next subsection.

Another reason to focus on predicting the median (and hence using MAE) is a compromise between different target variables that it offers. 
We have already mentioned that the expected travel time and expected speed are not related by a strict dependence, 
and therefore it is impossible to build one model that predicts both well. 
On the other hand, the relation is straightforward for the median: the median speed is exactly the length of the route divided by the median travel time. 
Thus, if we manage to predict the median travel time well, we can automatically achieve good results in predicting the median speed.

Hence, we recommend adopting MAE as the main optimization criteria, and focusing on predicting the median travel time.



\subsection{The best possible predictions}

Next, let us consider, what is the best possible MAE, that can be achieved in urban travel time prediction. 
Recall the restriction that no information about individual vehicles or individual drivers can be used, we need to output one prediction that fits everybody.

From Eq.(\ref{eq:mae}), the lower bound for MAE is \emph{zero}, which happens when $\hat{T}^{(d)} = T^{(d)}$ for all $d$ in $m$, i.e. the predictions are equal to the observed times for all the test trips. Suppose there is an oracle, that knows the future. Could the oracle achieve $\mathit{MAE} = 0$? 

We will analyze three different scenarios: 
\begin{enumerate}
\item road network consists of one segment, 
\item many segments, but all the vehicles follow the same route, 
\item many segments, and vehicles follow different routes. 
\end{enumerate}


\subsubsection{One segment}

For a start, suppose that the road network has only one segment. 
Then Eq. (\ref{eq:mae}) becomes 
\[
\mathit{MAE} = \frac{1}{m}\sum_{d=1}^{m} |\tau - t^{(d)}|, 
\]
where $\tau$ is the prediction for the segment, and $t^{(d)}$ is the travel time for trip $d$.
Since $\tau$ is the same for all trips, $\tau$ may be equal to $t^{(d)}$ for all $d$, and, in turn, MAE equal to zero only in one case, which is when the travel times for all the trips are the same. In such a case prediction is trivial. 
In any practical case, travel time of different vehicles over one segment varies, and, hence the minimum MAE cannot be zero. 

It can be proven that with one road segment the prediction $\tau = \mathit{median}(t^{(d)})$ minimizes MAE, since median minimizes the sum of absolute deviations (see e.g. \citet{Schwertman90} for a proof). 

\subsubsection{Many segments, single route}
\label{sec:wk} 

Now suppose the the road network has $k$ segments, but each trip follows the same route. 
As a result, the prediction is the same for all trips ${\cal{T}} = \sum_{i=1}^k \tau(r_i)$.
In this case Eq. (\ref{eq:mae}) becomes 
\[
\mathit{MAE} = 
\frac{1}{m}\sum_{d=1}^{m} |{\cal{T}} - T^{(d)}|,
\]
where ${\cal{T}}$ 
is the predicted travel time for the route, and $T^{(d)} $ 
is the observed travel time for trip $d$.

Following the same argument as in one segment case, the minimum MAE will be achieved, with the prediction ${\cal{T}} = \mathit{median}(T^{(d)})$. 
Since all the trips follow the same route, it is straightforward to compute the median of $T^{(d)}$, which is the median of the observed travel times for the route.

This scenario applies well to monitoring highways where traffic detectors are installed every few kilometers, and most of the vehicles are continuing through all the highway. 
However, this scenario is not very realistic in urban environment, where each vehicle may be following a different route. 
Hence, for urban environment we need to consider a more complex scenario.
	
\subsection{Many segments, different routes}


If a road network has many segments, and each vehicle follows a different route, 
for a trip $d$, the optimal prediction would be $\hat{T}^{(d)} = \mathit{median}(T^{(d)})$, as discussed in the previous section. 
The problem now is that we do not have enough observations of trips following the same route, or have no observations of such trips at all. 
Thus, we cannot compute $\mathit{median}(T^{(d)})$ directly. 

We could estimate the medians of each segment separately, and then add them up, but unfortunately, the sum of the medians is not generally equal to the median of the sum. 
Table \ref{tab:ex1} in the Appendix presents a proof by example. 
While the mean of the sum is, in fact, equal to the sum of the means, this does not help much. 
As we have discussed, travel data distribution is skewed, and in such a case the mean is not equal to the median, so we cannot easily reuse the mean for the median either.
Hence, our task resorts to modeling the median of a sum of random variables in the context of travel time estimation.

There have been many research attempts to model the sum of the medians for various data distributions under different assumptions (see e.g. \citet{Hall80}). 
We are not aware of existence of and analytical solution for a small number of variables, which is relevant to travel time prediction. 
Typically trips consist of a small number of segments, as it can be seen in Figure \ref{fig:k}, which presents an empirical distribution of trip lengths on a road network of 106 segments in St. Petersburg, recorded in November-December 2012. There is a peak at 44 segments, accounting for nearly 6\% of the trips, this is due to the main road in the network.
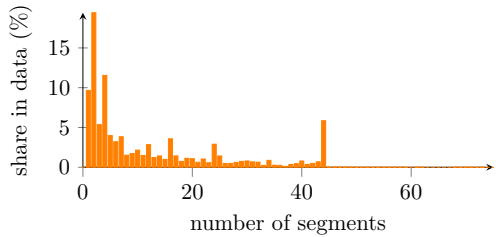
\begin{figure}
\begin{tikzpicture}[scale = 0.85]
\begin{axis}[width=8cm, height=4cm, xmin = 0, ymin = 0, xmax = 75,
axis lines=left,
xlabel = number of segments,
ylabel  =share in data (\%),
ybar, bar width = 1.8pt]
\addplot[fill,orange] table [x=step,y=count] {track_counts_NovDec.dat};
\end{axis}
\end{tikzpicture}
\caption{Empirical distribution of trip lengths according to the number of road segments (the total number of segments in the road network is 106).}
\label{fig:k}
\end{figure}

In summary, for an urban road network consisting of many segments, where vehicles follow different routes, we cannot obtain optimal travel time estimates (median over each route) directly.  Hence, we need an approach for estimating the median of the sum that would work well with small number of segments. 
In the next section we present our methodology for deriving a data driven approximation for travel time estimation in such a scenario.

\section{Methodology for aggregated travel time prediction}
\label{sec:method}

This section describes our methodology for travel time prediction in the urban settings with many road segments, and different routes.
The idea is at first to make estimates of mean and median travel times for each road segment. 
Then aggregate those estimates into a prediction for each route of interest. 
The simplest approach for obtaining the estimates is to take the mean and median over the last observed travel times. 
The remaining challenge is how to combine those estimates. We propose the following solution.

\subsection{Approach to solution}

Suppose a trip consists of $k$ segments. 
The total travel time is a sum of travel times over each segment  $T = \sum_{i=1}^k t_i$. 
We are looking for an estimate $\hat{T}$, which would minimize MAE. 
The solution is based on the following observations. 
\begin{enumerate}
\item When $k=1$ then MAE is minimized with the median over observed travel times $\hat{T} = \mathit{median(t_1)}$  (Sec. \ref{sec:criteria}).
\item When $k \rightarrow \infty$, assuming that travel times over each segment are identically and independently distributed (IID).
, $T$ approaches the normal distribution (Central Limit Theorem), 
thus $\mathit{median(T)} \approx \mathit{mean(T)}$, and $\mathit{mean(\sum_{i=1}^k t_i)} = \sum_{i=1}^k \mathit{mean(t_i)}$.
MAE is minimized with $\hat{T} = \sum_{i=1}^k \mathit{mean(t_i)}$. 
\item Observe that for positively skewed random variable $t$ the expected sum of the medians does not exceed the expected median of the sum, $\sum_{i=1}^k \mathit{median(t_i)} \leq \mathit{median(\sum_{i=1}^k  t_i)}$. 
\item For positively skewed random variable $t$, such as travel time (Section \ref{sec:datadistribution}), median of the sum is smaller than the mean of the sum \citet{Siegel01}, $\mathit{median(\sum_{i=1}^k  t_i)} \leq \mathit{mean(\sum_{i=1}^k  t_i)}$.
\end{enumerate}

Distributions of travel time over different segments may depend on time of day, resulting in different distributions at different times. We consider that for a given time, travel time observations can reasonably be assumed to be independent from each other.

The IID assumption does not exactly hold for the segments in our data, since the segments have different lengths, but, as we will see in the experimental analysis, this assumption  gives a good approximation. Moreover, when a route is sufficiently long, one can always partition it to segments of equal, and sufficiently large, 
length such that travel times over them are identically and independently distributed.

In summary, the median is smaller than the mean, but the median approaches the mean when the number of segments ($k$) becomes large. 
The solution for $k=1$ is the median, for $k \rightarrow \infty$ is the mean, hence, the solution for small positive $k$ should be in between of the median and the mean. 

\subsection{Solution - a combination of mean and median}

Based on these observations, our proposed solution is to model the optimal travel time estimate as a weighted average of the individual means and medians over segments:
\begin{equation}
\hat{T} = (1-w_k) \sum_{i=1}^k \mathit{median}(t_i) + w_k \sum_{i=1}^k \mathit{mean}(t_i),
\label{eq:estimate}
\end{equation}
where $w_k \in [0,1]$ is a weight. If sample sizes are sufficiently large to accurately estimate the median, this approach guarantees an optimal solution, proof can be found in the Appendix, Proposition \ref{pro:main}. But as we will see from the experiments in Section \ref{sec:experiments}, the proposed method shows good performance even for small sample sizes.

When $k=1$, $w_1=0$. When $k \rightarrow \infty$, $w_{\infty} \rightarrow 1$. 
For the rest of $k$ we model $w_k$ as a function of $k$ as follows. 
We randomly generate routes over the actual road network, 
sample travel times for these routes from the observed data, 
compute the median travel times and optimal $w_k$ for each route. 
This way we generate a semi-synthetic dataset, 
on which we can learn $w_k = f(k)$ using some machine learning method, that could capture a non-linear relation from data, for example, Artificial Neural Networks (ANN).

\subsection{Generating data for learning the weight function}
\label{sec:learnwk}

We need to generate a dataset, where the input variable is $k$ and the target variable is $w_k$. 
To generate one data point we:
\begin{enumerate}
\item select $k$ uniformly at random from a range $[1,k_{\mathit{max}}]$, where $k_{\mathit{max}}$ is the maximum length of a route, it depends on the considered road network;
\item randomly generate a track: select one segment at random, select the next segment uniformly at random from those segments connecting to the first one, continue until $k$ segments are selected;
\item for each segment on the generated track randomly pick one observation of travel time from the historical data, sum the selected travel times over $k$ segments, repeat this $h$ times to obtain $h$ trips, 
\item compute the true median over $h$ trips, 
\item let $w_k$ run from $0$ to $1$ (grid search), and select $w_k$, which gives the minimum absolute deviation of the estimate computed as in Eq. (\ref{eq:estimate}) from the true median, computed in the previous step.
\end{enumerate}
This procedure gives one data point. 
We generate $N$ such data points, and use them for modeling $w_k$ as a function of $k$. 




\subsection{Learning the weight function}
\label{sec:wk}

Once we have generated a dataset for relating $k$ and $w_k$, 
we can proceed in two ways: we can construct a look-up table, where for each $k$ we can list a corresponding $w_k$, or
we can find a functional form $w_k = f(k)$ using a machine learning method, for example, ANN.
As an example, illustrating that a nice functional form of $w_k$ exists, 
Figure \ref{fig:residuals} plots $w_k = f(k)$ learned on synthetic data sampled from log-normal distribution $ln{\cal N}(0,s)$. 
We can see that different data distributions give different functions, but the learned functions give accurate approximations, as tested on out-of-sample data.
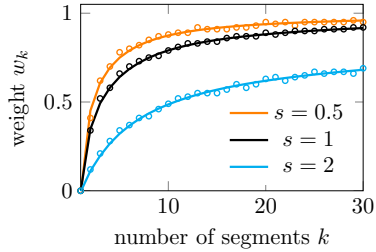
\begin{figure}[t]
\begin{tikzpicture}[scale = 0.85]
\begin{axis}[width=6cm, height=4.5cm, xmin = 1, ymin = 0, xmax=30,ymax=1.05,
ylabel = weight $w_k$,
xlabel = number of segments $k$,
legend pos= south east,
legend style={draw=none}]
\addplot[orange,line width = 1pt]  table[x=k, y=hy05] {lognorm.dat};
\addplot[black,line width = 1pt]  table[x=k, y=hy1] {lognorm.dat};
\addplot[cyan,line width = 1pt]  table[x=k, y=hy2] {lognorm.dat};
\addplot[orange,mark=o, only marks, line width = 0.6pt,mark size=1.2pt]  table[x=k, y=y05] {lognorm.dat};
\addplot[black,mark=o, only marks, line width = 0.6pt,mark size=1.2pt]  table[x=k, y=y1] {lognorm.dat};
\addplot[cyan,mark=o, only marks, line width = 0.6pt,mark size=1.2pt]  table[x=k, y=y2] {lognorm.dat};
\legend{$s=0.5$,$s=1$,$s=2$};
\end{axis}
\end{tikzpicture}
\caption{$w_k = f(k)$ learned with ANN. Solid lines denote estimated functions, circles denote out-of-sample test data.}
\label{fig:residuals}
\end{figure}
While the weight function is different for different distributions (with different standard deviations), we can consider that for a given road network of interest the distribution is fixed, and therefore, one estimated weight function can be learned and applied to that road network for a given time period.   

\subsection{Travel time prediction scenario}

After learning the function\\ $w_k = f(k)$, we recommend the following procedure for travel time prediction.
\begin{enumerate}
\item Estimate the mean and the median travel times for each segment on the road network. 
\item For each trip of interest compute an aggregated prediction using $w_k$, as specified in Eq. (\ref{eq:estimate}).
\end{enumerate} 

Separate $w_k$ can be learned and used, for instance, for different times of the day. 
In case more complex machine learning approaches are used for travel time prediction, including, e.g., time of the day, holiday, weather information as input features, we suggest building two models per segment: for predicting the mean, and predicting the median. 
Then, for a trip of interest combine the predictions using Eq. \ref{eq:estimate} with a fixed function $w_k = f(k)$. 

\section{Experimental analysis}
\label{sec:experiments}

We analyze the performance of the proposed approach with a case study from St. Petersburg city. 
The goal is to compare the performance of the new approach to the performance of two baselines: sum of medians and sum of means.
\subsection{Dataset}


We use one year (2012) of GPS data, collected via a mobile application Yandex.Navigator. 
The dataset covers an interconnected network of streets in St. Petersburg. 
The network consists of 106 road segments. 
$1\,132\,277$ trips are recorded in the dataset, covering over $14$M records. The dataset is extensive in number of records and the time covered, and, we believe, is as representative of the traffic population in St. Petersburg, as it could be. 
Figure \ref{fig:k} presents an empirical distribution of trips in terms of number of segments. 

\subsection{Experimental protocol}

To minimize possibilities of overfitting, only January-February data is used as training data for generating route data in order to learn the weight function, 
and March-December data is used for testing the predictive performance.

We operate in discrete time steps $\Delta$, over which we collect and summarize input data, and for which we make predictions. 
For example, if $\Delta = 10$min, and now is 14:00 o'clock, we would estimate the mean and the median travel time on each segment from the time interval 13:50-14:00, and use it for making predictions for the time interval 14:00-14:10. 

A trip is included into mean and median estimation after it has finished. That is, if, for instance, a trip starts at 13:48 and ends at 13:55, 
it is included into the estimation interval 13:50-14:00. Predictions are based on the data accumulated before a test trip starts. That is, if a test trip starts 14:07 and ends 14:15, data from the interval 13:50-14:00 is used for making predictions for the whole test trip.  

We compare the performance of the proposed combined approach (COM) with two baselines: the sum of means (SMN) and the sum of medians (SMD). 
Where technically possible, we also show the performance of the true median prediction (MED), which gives theoretically optimal prediction.

The true median (MED) is the theoretically optimal solution, which is only possible if we have many observed trips following the same route. 
It is possible to estimate MED on particular routes, when many vehicles follow the same route, but it is not feasible to estimate MED for most of the routes on the road network. 
The next best we can do is to approximate MED by combining the medians and the means over individual segments (COM), which is our proposed solution, and is expected to give an accuracy close to that of MED. Finally, SMN and SMD are two baseline approaches. 
SMD is expected to perform well on routes consisting of a small number of segments, and SMN is expected to perform well on the routes consisting of a large number of segments. 

For COM we estimate the optimal combination weight $w_k$ by letting $w_k$ run from $0$ to $1$ and selecting the one, which gives the minimum estimation error, as described in Sec. \ref{sec:wk}. 

We standardize the prediction errors by the trip lengths, such that the figures are comparable across different samples, as
\begin{equation}
\mathit{MAE}^\star =  \sum_{d=1}^{m} |\hat{T}^{(d)} - T^{(d)}|/L^{(d)},
\end{equation}
where $L^{(d)}$ is the length of trip $d$ (in km), $L^{(d)} = \sum_{i=1}^{k_d}l_{r_{\delta_i^{(d)}}}$. Note, that the standardized $\mathit{MAE}^\star$ relates to $\mathit{MAE}$ in Eq.~\ref{eq:mae} as $\mathit{MAE}^\star = m\mathit{MAE}/\sum_{d=1}^{m}L^{(d)}$.

For easier visual interpretation and comparison across different aggregation times we plot $\mathit{MAE}$ relative to the performance of the sum of the means baseline (SMN) that is currently being used in practice as the state of the art approach. Relative $\mathit{MAE}$ is $\mathit{MAE}$ of the approach of interest divided by $\mathit{MAE}$ of the baseline. If the relative $\mathit{MAE}$ is smaller than one that means that the approach of interest is performing better than the baseline. 
Since relative $\mathit{MAE}$ incorporates the baseline, from the result we know how good the methods are from the global perspective. 


\subsection{Results: many segments, one route}

The goal of this experiment is to verify, whether the proposed combined approach provides (COM) a better estimate for travel times than the baselines (SMN and SMD). 
In this experiment we use only a subset of data, consisting of trips over the main road (nearly 6\% of the trips in the dataset), and refer to this subset route as Route44. 
Since this route is popular, we have a number of traces following this particular route, and thus we are able to compute the theoretically optimal prediction, which is the median travel time over the whole route (MED).
Thus, we can also investigate, how close the proposed approach (COM) comes to the theoretically optimal prediction (MED). 

While the main route consists of 44 segments (a fixed $k=44$), we can also estimate $w_k$ for any $k\leq 44$ by considering a shorter sub-section of the main road, for instance, when $k=1$ we would consider only the first segment of the main route, and when $k=2$ we would consider the first two segments. 
For each $k$ we select one sub-segment, which starts at segment \#1 and ends at segment \#$k$.

%

\subsubsection{Learning the weight function}

Figure \ref{fig:44sen} presents an estimation error as a function of different possible weights $w_k$. 
We can clearly see that there is a global minimum in each case, suggesting that there is a weight for combining the sum of means and and the sum of medians in an optimal way. 
Since the minimum of the solid line overlaps with the dashed line, we can conclude that the combination approach can give the minimum error solution, that can be achieved by knowing the true median. 

\newcommand{\skw}{3.6cm}
\begin{figure}
    \centering
    \begin{minipage}{.61\textwidth}
\begin{tikzpicture}[scale=0.85]
\begin{axis}[name=plot1,
width=\skw, height=\skw, xmin = 0, xmax = 1,
 xtick={0,0.5,1},
ylabel = $\mathit{MAE}$,
xlabel = $w_k$, title = {$k=2$}]
\addplot[red,dashed,line width = 0.7pt]  table[x=wk, y=EE2] {out_route44_EEmed.dat};
\addplot[black,line width = 2pt]  table[x=wk, y=EE2] {out_route44_errors.dat};
\end{axis}
\begin{axis}[name=plot2,
 at=(plot1.right of north east), anchor=left of north west,
width=\skw, height=\skw, xmin = 0, xmax = 1,
 xtick={0,0.5,1},
xlabel = $w_k$, title = {$k=5$}]
\addplot[red,dashed,line width = 0.7pt]  table[x=wk, y=EE5] {out_route44_EEmed.dat};
\addplot[black,line width = 2pt]  table[x=wk, y=EE5] {out_route44_errors.dat};
\end{axis}
\begin{axis}[name=plot3,
 at=(plot2.right of north east), anchor=left of north west,
width=\skw, height=\skw, xmin = 0, xmax = 1,
 xtick={0,0.5,1},
xlabel = $w_k$, title = {$k=44$}]
\addplot[red,dashed,line width = 0.7pt]  table[x=wk, y=EE44] {out_route44_EEmed.dat};
\addplot[black,line width = 2pt]  table[x=wk, y=EE44] {out_route44_errors.dat};
\end{axis}
\end{tikzpicture}
\caption{Estimation error as a function of different weights $w_k$ (solid line), and true median error (dashed).}
\label{fig:44sen}
    \end{minipage}%
    \hfill
    \begin{minipage}{0.35\textwidth}
\begin{tikzpicture}[scale = 0.85]
\begin{axis}[width=5.5cm, height=4cm, xmin=1, xmax=44,ymin = 0, ymax = 1,
ylabel = weight $w_k$, xlabel = number of segments $k$]
\addplot[orange,line width = 1pt,mark = *, mark size=1pt]  table[x=kk, y=wk] {out_route44_w.dat};
\end{axis}
\end{tikzpicture}
\caption{Learned weights $w_k$ for trips of length $k$ on Route44.}
\label{fig:route44_w}
    \end{minipage}
\end{figure}



The resulting weights $w_k$ for Route44, estimated on the training data (Jan-Feb 2012), is presented in Figure \ref{fig:route44_w}. 
The learned weights are not monotonous over increasing $k$, perhaps, due to varying length of segments across the route.  
Next, we will use this $w_k$ for predicting travel time on unseen data.

\subsubsection{Prediction accuracy}

Figure~\ref{fig:route44_err} plots the prediction errors of four alternative approaches. 
The discretization in this experiment is $\Delta = 120$ min, thus, the prediction horizon is 0-120 min. 
The discretization step is chosen in this experiment such that several samples are available in time slot in each segment.

We can see that the proposed approach COM clearly performs better than the baselines SMN and SMD. At small $k$ COM performs slightly better than SMD, and SMN performs much worse, and at larger $k$ COM performs slightly better than SMN, and SMD performs much worse, as expected. 
Hence, COM combines the advantages of the two baselines. Moreover, COM performs as good as the theoretically optimal MED, which is a good news, since MED is rarely feasible to compute in practice, and COM shows to be a good approximation. 
Interestingly, sometimes COM performs slightly better than MED.
That can be explained by small sample sizes from which MED is estimated, in which case the sample median is not as precise as the weighted average obtained by COM using the learned weights $w_k$.
\begin{figure}
\centering
\begin{minipage}{.5\textwidth}
\begin{tikzpicture}[scale = 0.85]
\begin{axis}[width=7cm, height=5.5cm, xmin = 0, xmax = 44, 
ylabel = relative MAE,
xlabel = number of segments $k$,
legend style={font=\tiny,line width=.5pt,mark size=.6pt},
legend pos= south east,
legend style ={draw=none}]
\addplot[cyan,line width = 1.5pt]  table[x=kk, y=SMN] {mean_median_errors_normed.csv};
\addplot[orange,dashed,line width = 2pt]  table[x=kk, y=SMD] {mean_median_errors_normed.csv};
\addplot[black,line width = 1pt]  table[x=kk, y=MED] {mean_median_errors_normed.csv};
\addplot[green,dotted,line width = 2pt]  table[x=kk, y=COM] {mean_median_errors_normed.csv};
\legend{SMN sum means, SMD sum medians, MED true median, COM combined};
\end{axis}
\end{tikzpicture}
\caption{Relative prediction errors on Route44 on test data Mar-Dec 2012.}
\label{fig:route44_err}
\end{minipage}
\hfill
\begin{minipage}{.45\textwidth}
\begin{tikzpicture}[scale = 0.85]
\begin{axis}[width=5.5cm, height=4cm, xmin=1, xmax=70,ymin = 0, ymax = 1,
ylabel = weight $w_k$,
xlabel = number of segments $k$]
\addplot[orange,line width = 1pt,mark = *, mark size=1pt]  table[x=kk, y=wk] {out_w.dat};
\end{axis}
\end{tikzpicture}
\caption{Learned weights $w_k$ on all routes, estimated on Jan-Feb 2012 data.}
\label{fig:w}
\end{minipage}
\end{figure}


\subsection{Results: many segments, different routes}

Next, we analyze how the proposed approach performs on a road network, consisting of many segments and different routes. 

\subsubsection{Learning the weight function}

Figure \ref{fig:w} presents combination weights $w_k$, learned using artificial neural network with four hidden layers. 
The weights monotonically increase with $k$, as expected, since the median of the sum approaches the mean of the sum, as discussed in Sec. \ref{sec:method}.
We see low weights at small $k$, which means that when a trip consists of small number of segments, the sum of the medians SMD dominates. 
When the number of segments per trip reaches about $10$,  the sum of means SMN starts to dominate. 

\subsubsection{Predictive performance}

Figure \ref{fig:sen} plots predictive performance month-by-month for different discretization steps $\Delta$. The prediction horizon corresponds to the discretization step $\Delta$, i.e. if $\Delta = 10$ min, the prediction horizon is 0-10 min.
We see that the proposed combination approach COM consistently outperforms both baselines at different discretization steps (and prediction horizons) over the course of year.

\newcommand\htt{4cm}
\newcommand\wdd{5.5cm}
\newcommand\csl{0.8}
\newcommand\lw{1pt}
\newcommand\ymin{0.97}
\newcommand\ymax{1.05}
\begin{figure*}[t]
\centering
\begin{tikzpicture}
\matrix{
\begin{axis}[scale = \csl,width=\wdd, height=\htt, ymin = \ymin, ymax = \ymax,
legend columns=-1,
legend entries={SMN,SMD,COM},
legend to name=named,
ylabel = relative $\mathit{MAE}$, xticklabels={,,}, xtick={3,...,12},
title = {discretization 10 min}]
\addplot[cyan,line width = 1.5pt]  table[x=mn, y=SMN] {mean_median_errors_10.dat};
\addplot[orange,dashed,line width = 2pt]  table[x=mn, y=SMD] {mean_median_errors_10.dat};
\addplot[green,dotted,line width = 2pt]  table[x=mn, y=COM] {mean_median_errors_10.dat};
\end{axis}
&
\begin{axis}[scale = \csl,width=\wdd, height=\htt, ymin = \ymin, ymax = \ymax,  
xticklabels={,,}, xtick={3,...,12},
title = {discretization 20 min}]
\addplot[cyan,line width = 1.5pt]  table[x=mn, y=SMN] {mean_median_errors_20.dat};
\addplot[orange,dashed,line width = 2pt]  table[x=mn, y=SMD] {mean_median_errors_20.dat};
\addplot[green,dotted,line width = 2pt]  table[x=mn, y=COM] {mean_median_errors_20.dat};
\end{axis}
&
\begin{axis}[scale = \csl,width=\wdd, height=\htt, ymin = \ymin, ymax = \ymax,
xticklabels={,,}, xtick={3,...,12},
title = {discretization 30 min}]
\addplot[cyan,line width = 1.5pt]  table[x=mn, y=SMN] {mean_median_errors_30.dat};
\addplot[orange,dashed,line width = 2pt]  table[x=mn, y=SMD] {mean_median_errors_30.dat};
\addplot[green,dotted,line width = 2pt]  table[x=mn, y=COM] {mean_median_errors_30.dat};
\end{axis}\\
\begin{axis}[scale = \csl,width=\wdd, height=\htt, ymin = \ymin, ymax = \ymax,
ylabel = relative $\mathit{MAE}$, 
xticklabels={Mar,Apr,May,Jun,Jul,Aug,Sep,Oct,Nov,Dec}, x tick label style={rotate=90,anchor=east}, xtick={3,...,12},
title = {discretization 40 min}]
\addplot[scale = \csl,cyan,line width = 1.5pt]  table[x=mn, y=SMN] {mean_median_errors_40.dat};
\addplot[orange,dashed,line width = 2pt]  table[x=mn, y=SMD] {mean_median_errors_40.dat};
\addplot[green,dotted,line width = 2pt]  table[x=mn, y=COM] {mean_median_errors_40.dat};
\end{axis}
&
\begin{axis}[scale = \csl,width=\wdd, height=\htt, ymin = \ymin, ymax = \ymax,
xticklabels={Mar,Apr,May,Jun,Jul,Aug,Sep,Oct,Nov,Dec}, x tick label style={rotate=90,anchor=east}, xtick={3,...,12},
title = {discretization 50 min}]
\addplot[cyan,line width = 1.5pt]  table[x=mn, y=SMN] {mean_median_errors_50.dat};
\addplot[orange,dashed,line width = 2pt]  table[x=mn, y=SMD] {mean_median_errors_50.dat};
\addplot[green,dotted,line width = 2pt]  table[x=mn, y=COM] {mean_median_errors_50.dat};
\end{axis}
&
\begin{axis}[scale = \csl,width=\wdd, height=\htt, ymin = \ymin, ymax = \ymax,
xticklabels={Mar,Apr,May,Jun,Jul,Aug,Sep,Oct,Nov,Dec}, x tick label style={rotate=90,anchor=east}, xtick={3,...,12},
title = {discretization 60 min}]
\addplot[cyan,line width = 1.5pt]  table[x=mn, y=SMN] {mean_median_errors_60.dat};
\addplot[orange,dashed,line width = 2pt]  table[x=mn, y=SMD] {mean_median_errors_60.dat};
\addplot[green,dotted,line width = 2pt]  table[x=mn, y=COM] {mean_median_errors_60.dat};
\end{axis}\\
};
\end{tikzpicture}
\\
{\footnotesize
\ref{named}
}
\caption{Testing errors on the whole route network month-by-month.}
\label{fig:sen}
\end{figure*}
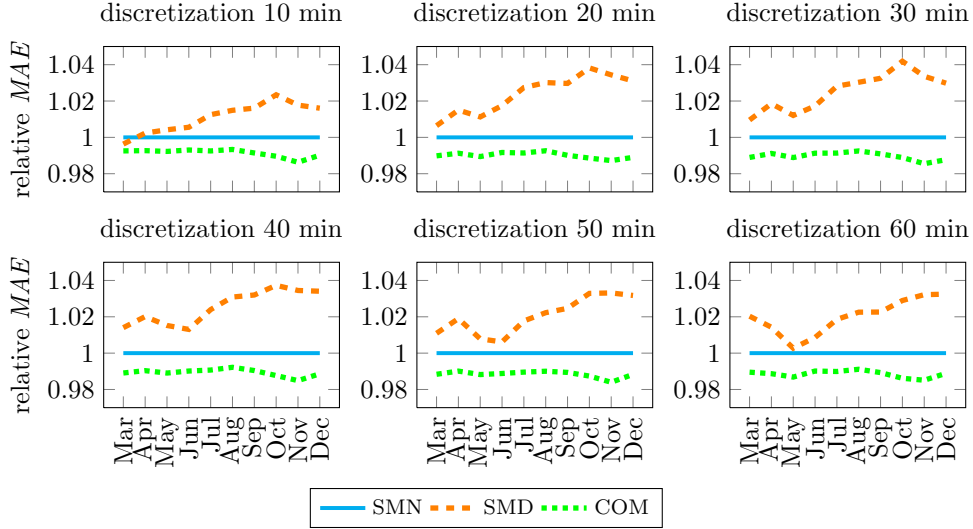

\section{Conclusion}
\label{sec:conclusion}

In this study we analyzed optimization criteria for travel time prediction in urban environments, where vehicles follow many different routes.
We proposed a methodology for aggregated travel time prediction, which interactively combines  the mean and median estimates of travel times over individual trip segments into a single prediction for the whole trip. Experimental results demonstrated that the proposed approach consistently outperforms the current baselines. 

Based on the results, we recommend using the proposed combination of the sum of means and sum of medians of individual road segments for constructing aggregated predictions.

This study opens several interesting directions for further research. 
We have focused on optimizing the mean absolute error of predictions. 
We have argued that while this quantitative optimization criteria may occasionally permit large errors, it favors the scenario to predict accurately for as many users as possible. 
From the practical perspective, it would be interesting to consider, and possibly integrate, multiple optimization criteria. 
For instance, a designer may want the system to focus on accurately predicting travel time over longer routes, while errors on short routes do not matter that much.  
Another interesting extension would be to consider asymmetric costs of errors, where, for instance, overprediction of travel time is tolerated better than underprediction.


\bibliographystyle{abbrvnat}
\bibliography{bib_trafficjams}

\appendix

\section{Proofs}

\begin{prop} 
Sum of the medians is not necessarily equal to the median of the sum.
\end{prop} 

\begin{proof}
Proof by example. Example data is presented in Table \ref{tab:ex1}. From the table $\mathit{median}(\mathit{sum}(r)) = 24$, but $\mathit{sum}(\mathit{median}(r)) = 5 + 7 + 8 = 20$.
\begin{table}
\caption{Example data (medians of the columns are denoted in bold).}
\centering
\begin{tabular}{l|ccc|c}
	& $r_1$ & $r_2$ & $r_3$ & $\mathit{sum}(r)$\\
\hline	
$t_1$ & 1 & \textbf{7} & \textbf{8} & 16\\
$t_2$ & 3 & 3 & 11 & 17\\
$t_3$ & \textbf{5} & 2 & 17 & \textbf{24}\\
$t_4$ & 10 & 9 & 6 & 25\\
$t_5$ & 20 & 11 & 3 & 34\\
\hline
\end{tabular}
\label{tab:ex1}
\end{table}
\end{proof}

\begin{prop}
Given the weight parameter $w_k \in [0,1]$, the following model exactly estimates the median of the sum for positively skewed identically and independently distributed $t$
\[
\mathit{median}(\sum_{i=1}^k t_i) = (1-w_k) \sum_{i=1}^k \mathit{median}(t_i) + w_k \sum_{i=1}^k \mathit{mean}(t_i),
\]
where $t_t$ is the travel time over segment $i$, $k$ is the number of segments, $w_k \in [0,1]$ is the weight parameter.
\label{pro:main}
\end{prop}

\begin{proof}
The expression can be rearranged into
\[
w_k = \frac{\mathit{median}(\sum_{i=1}^k t_i) -  \sum_{i=1}^k \mathit{median}(t_i) }{\sum_{i=1}^k \mathit{mean}(t_i) - \sum_{i=1}^k \mathit{median}(t_i) }.
\]
From the definition of mean the following identity holds $\sum_{i=1}^k \mathit{mean}(t_i)  = \mathit{mean}(\sum_{i=1}^k t_i)$. 
From Sec. \ref{sec:method} $\mathit{median}(\sum_{i=1}^k t_i) \leq \sum_{i=1}^k \mathit{mean}(t_i)$, and $\sum_{i=1}^k \mathit{median}(t_i) \leq \mathit{median}(\sum_{i=1}^k t_i)$. After plugging in these inequalities into the expression above we get $0 \leq w_k \leq 1$, which means that  given the right $w_k$ we can approximate the median of the sum as a combination of the sum of the medians and the sum of the means. 
\end{proof}

\begin{prop} 
Let $F$ be monotone function,\\ $T^A \sim F( {\cal N}(\mu_A,\sigma_A^2))$ and $T^B \sim F( {\cal N}(\mu_B,\sigma_B^2))$ be two independent random variables. 
Then $p(T^A > T^B) > \frac{1}{2} \iff \mathit{median}(T^A) > \mathit{median}(T^B)$
\end{prop} 
\label{pro:monotone}
\begin{proof}
First, let us prove the proposition in a special case, where $F$ is the identity function, i.e., $T^A$ and $T^B$ are normally distributed.

In this case $\mathit{median}(T^A) = \mu_A$, $\mathit{median}(T^B) = \mu_B$. $p(T^A > T^B) = p(T^A - T^B > 0)$. Denote $T = T^A - T^B$.
Since $T^A$ and $T^B$ are independent, $T$ is also normally distributed with mean $\mu = \mu_A - \mu_B$ and standard deviation $\sigma = \sqrt{\sigma_A^2 + \sigma_B^2}$.
$p(T > 0) = \frac{1}{2} + \frac{1}{2} \mathit{erf}(\frac{\mu}{\sigma \sqrt{2}})$. Thus, $p(T > 0) > \frac{1}{2}$ is equivalent to $\mathit{erf}(\frac{\mu}{\sigma \sqrt{2}}) > 0$, which in turn is equivalent to $\mu > 0$, because the error function is odd. Since $\mu = \mu_A - \mu_B$, $p(T^A > T^B) > \frac{1}{2} \iff \mu_A > \mu_B$. 
That proves the proposition in the special case.

Now let us prove the general case. For any monotone function $F$ there exists the inverse $F^{-1}$ which is also monotone. 
Let us denote $t^A = F^{-1}(T^A)$, $t^B = F^{-1}(T^B)$. Since $F^{-1}$ is monotone, $T^A > T^B \iff t^A > t^B$. 
From this fact two conclusions can be drawn.

First, $p(T^A > T^B) = p(t^A > t^B)$.

Second, $p(t^A > F^{-1}(\mathit{median}(T^A))) = p(T^A > \mathit{median}(T^A)) = \frac{1}{2}$, in other words,  $\mathit{median}(t^A) = F^{-1}(\mathit{median}(T^A))$. 
The same holds for $t^B$ and $T^B$.

Now, if $\mathit{median}(T^A) > \mathit{median}(T^B)$, then $\mathit{median}(t^A) > \mathit{median}(t^B)$. $t^A$ and $t^B$ are normally distributed variables, so as proved above, $p(t^A > t^B) > \frac{1}{2}$. But since $p(T^A > T^B) = p(t^A > t^B)$, then $p(T^A > T^B) > \frac{1}{2}$.

Conversely, if $p(T^A > T^B) > \frac{1}{2}$ then $p(t^A > t^B) > \frac{1}{2}$, hence $\mathit{median}(t^A) > \mathit{median}(t^B)$ and $\mathit{median}(T^A) > \mathit{median}(T^B)$.
\end{proof}


%
%


\end{document}